\documentclass{article}
\usepackage{spconf,amsmath,graphicx}
\usepackage{bbold}
\usepackage{amssymb,bm}
\usepackage{float}
\usepackage{colortbl}
\usepackage{algorithm}
\usepackage[noend]{algpseudocode}
\usepackage{subfigure}
\usepackage[english]{babel}
\usepackage{amsthm}
\usepackage[dvipsnames]{xcolor}
\definecolor{mygreen}{RGB}{0, 199, 0}
\definecolor{myorange}{RGB}{250, 100, 0}
\definecolor{myred}{RGB}{200, 0, 0}
\definecolor{myblue}{RGB}{30, 144, 255}
\definecolor{mylightskyblue}{RGB}{135, 206, 250}
\definecolor{myskyblue}{RGB}{0, 191, 255}
\definecolor{mypowderblue}{RGB}{176, 196, 222}
\newtheorem{theorem}{Theorem}[section]
\newtheorem{proposition}[theorem]{Proposition}

\makeatletter
\def\algbackskip{\hskip-\ALG@thistlm}
\makeatother
\title{Regularized EM algorithm}
\name{Pierre Houdouin$^\dagger$, Esa Ollila$^\star$, Fr\'ed\'eric  Pascal$^\dagger$}

\address
{$^\dagger$ Université Paris-Saclay, CNRS, CentraleSupélec, Laboratoire des signaux et systèmes,\\ 91190, Gif-sur-Yvette, France. \\
$^\star$
Department of Signal Processing and Acoustics, 
Aalto University, Finland \\
}

\begin{document}
%
\maketitle
\begin{abstract}
Expectation-Maximization (EM) algorithm is a widely used iterative algorithm for computing (local) maximum likelihood estimate (MLE). It can be used in an extensive range of problems, including the clustering of data based on the Gaussian mixture model (GMM). 
Numerical instability and convergence problems may arise 
in situations where the sample size is not much larger than the data dimensionality. In such low sample support (LSS) settings, the covariance matrix update in the EM-GMM algorithm may become singular or poorly conditioned, causing the algorithm to crash. 
On the other hand, in many signal processing problems, a priori information can be available indicating certain structures for different cluster covariance matrices. In this paper, we present a regularized EM algorithm for GMM-s that can make efficient use of such prior knowledge as well as cope with LSS situations. 
The method aims to maximize a penalized GMM likelihood where  regularized estimation may be used to ensure positive definiteness of covariance matrix updates and shrink the estimators towards some structured target covariance matrices. We show that the theoretical guarantees of convergence hold, leading to better performing EM algorithm for structured covariance matrix models or with low sample settings.
\end{abstract}
\section{Introduction}
\label{sec:intro}

Expectation-Maximization algorithm is a widely used iterative algorithm for finding the (local) maxima of the likelihood with incomplete data. It iteratively estimates the unknown parameters of the model by increasing the expected likelihood of the complete data conditioned over the observed data and current estimates of the parameters. The EM algorithm was proposed in \cite{Dempster2022Maximum} 
where the authors showed that at each iteration, the likelihood of the samples increases at least as much as the conditional expected likelihood. This likelihood maximization method is very suitable for dealing with mixture models. 
Work in \cite{Richard2022Mixture} considers EM algorithm 
for mixtures of exponential families. The problem of the Gaussian mixture model (GMM) has been treated in \cite{Guorong2001EM} and later extended by \cite{Ingrassia2012Studies} to $t$- distributions in order to cope with heavy-tailed data and outliers. More recently, generalization to mixtures of elliptical distributions has been developed for clustering applications in \cite{roizman2019flexible} as well as for classification applications (see \cite{houdouin2022robust}). Then, \cite{Teimour2021EM}  proposed an adaptation of the EM algorithm to deal with skewed normal distributions.

In signal processing problems,
dimension of the data, $m$, is often very high while the number of available samples, $n$, is low or of similar magnitude as $m$. In such low sample support (LSS) settings, convergence issues may arise when applying the EM algorithm for GMM. Namely, the update of the covariance matrix at each step is no longer necessarily invertible or poorly conditioned, leading the M-step to be ill-defined. To avoid this problem, \cite{Yi2015Regularized} proposes a regularized version of the EM algorithm for GMM based on resampling. Although the developed method ensures the invertibility of covariance matrix update, it does not exploit the possible underlying structure of the covariance matrix. More recently, \cite{Cai2017CHIME} proposed the CHIME algorithm to deal with high dimensional GMM-s. It relies on the estimation of discriminant vectors $\boldsymbol{\Sigma}^{-1}(\boldsymbol{\mu}_i-\boldsymbol{\mu}_j)$ then combined with Fisher Discriminant Analysis.

Regularized covariance matrix estimation \cite{Ying2014Regularized,pascal2014generalized,Ollila2019Optimal,ollila2021shrinking,yi2020shrinking} is a popular technique to cope with low sample support (LSS) settings. 
In this paper, we present a regularized EM algorithm where we maximize a penalized version of the GMM likelihood that leads to better conditioned regularized estimators  of the cluster covariance matrices as well as better clustering accuracy in LSS settings. We show that theoretical convergence guarantees hold and perform simulation experiments to validate the results. 

The paper is structured as follows. Section \ref{sec:2} recalls the GMM and provides the main derivations of this work for penalized GMM. Section \ref{sec:3} contains experiments on synthetic data. 
Conclusions, remarks, and perspectives are drawn in Section \ref{sec:conclusions}.

\section{Regularized EM algorithm}
\label{sec:2}

Let us assume that each observation $\mathbf{x}_i \in \mathbb{R}^m$ is drawn from a GMM. Each cluster has its own mean vector $\boldsymbol{\mu}_k \in \mathbb{R}^m$ and symmetric positive definite (SPD) $m \times m$ covariance matrix $\boldsymbol{\Sigma}_k$. We then have the following probability density function for $\mathbf{x}_i$ with priors (mixing proportions) $\pi_1,...,\pi_K \in [0,1]$ that verify $\sum_k \pi_k=1$, mean vectors  $\boldsymbol{\mu}_1,...,\boldsymbol{\mu}_K$ and covariance matrices $\mathbf{\Sigma}_1,...,\mathbf{\Sigma}_K$: 
$$
f(\mathbf{x}_i |\boldsymbol{\theta} ) =  (2\pi)^{-\frac{m}{2}} \sum_{k=1}^K \pi_k \left|\mathbf{\Sigma}_k \right|^{-\frac{1}{2}}  e^{ -\frac{1}{2} (\mathbf{x}_i-\boldsymbol{\mu}_k)^\top \mathbf{\Sigma}_k^{-1} (\mathbf{x}_i-\boldsymbol{\mu}_k)}
$$
where  $\boldsymbol{\theta} = (\pi_1,...,\pi_K,\boldsymbol{\mu}_1,...,\boldsymbol{\mu}_K,\mathbf{\Sigma}_1,...,\mathbf{\Sigma}_K)$ collects all unknown parameters. 

A priori knowledge of the structure of the cluster covariance matrices can be brought into the GMM estimation problem in forms of fixed SPD target $m \times m$ covariance matrices $\mathbf{T}_k$, $k=1,\ldots,K$. Namely, the covariance matrix of $k$th cluster is assumed to be close (but not identical) to a prespecified target matrix $\mathbf{T}_k$. 
To ensure that such structure is exploited, we penalize the likelihood with the Kullback-Leibler divergence between each $\mathbf{\Sigma}_k$ and $\mathbf{T}_k$,  defined as \cite{Ying2014Regularized}
\[
\Pi_{\textup{KL}}(\boldsymbol{\Sigma}_k,\mathbf{T}_k) =  \frac 1 2 \big(\mathrm{tr}(\boldsymbol{\Sigma}_k^{-1} \mathbf{T}_k)  - \log | \boldsymbol{\Sigma}_k^{-1} \mathbf{T}_k | - m \big).
\]

Let $ \mathbf{X} = (\mathbf{x}_1 \ \ \cdots \mathbf{x}_n)$ denote the data matrix of $n$ i.i.d. samples from the GMM model. Then the log-likelihood of $\boldsymbol{\theta}$ based on the data is  
$$
\ell(\boldsymbol{\theta} \vert \mathbf{X} ) = \sum_{i=1}^n \log f( \mathbf{x}_i | \boldsymbol{\theta}) 
$$ 
and the aim is then to find a (local) maximizer of the penalized GMM likelihood function:
$$ 
\ell_{\boldsymbol{\eta}}( \boldsymbol{\theta} \vert \mathbf{X} )=  \ell(\mathbf{X}|\boldsymbol{\theta}) - \sum_{k=1}^K \eta_k \Pi_{\textup{KL}}(\mathbf{\Sigma}_k,\mathbf{T}_k)
$$
where $\eta_1,...,\eta_K \geq 0 $ denote the penalization parameters for each cluster specified by the user.    
Let $\mathbf{z} = (z_1, \ldots,z_n)^\top$ represent the vector of latent unobserved labels. Thus  
$z_i  \in \{1,\ldots,K\}$ specifies to which population $i$th observation belongs to. However, this information is missing or is unobservable. The data $(\mathbf{X},\mathbf{z})$ is considered to as complete data, while the observed data $\mathbf{X}$ are referred to as incomplete data.  
Let $\log f(\mathbf{X}, \mathbf{z} \vert \boldsymbol{\theta})$ denote the log-likelihood function of the complete data.

\begin{proposition}
  Starting with an initial value $\boldsymbol{\theta}^{(0)}$, the iterative ($t=1,2, \ldots$) maximization of penalized conditional expected likelihood of the complete data   
  $$ Q( \boldsymbol{\theta} \vert \boldsymbol{\theta}^{(t)}) =  \mathbb{E}_{\mathbf{z}|\mathbf{X}, \boldsymbol{\theta}^{(t)}}\left[ \log f(\mathbf{z}, \mathbf{X}|\boldsymbol{\theta})\right] - \sum_{k=1}^K \eta_k \Pi_{\textup{KL}}(\mathbf{\Sigma}_k,\mathbf{T}_k)
  $$ 
  leads to sequence that ascents the penalized likelihood, i.e.,  $\ell_{\boldsymbol{\eta}}( \boldsymbol{\theta}^{(t+1)}  \vert \mathbf{X} ) \geq \ell_{\boldsymbol{\eta}}( \boldsymbol{\theta}^{(t)}  \vert \mathbf{X} )$. 
\end{proposition}
\begin{proof} First note that 
\begin{align*}
&\ell_{\boldsymbol{\eta}}( \boldsymbol{\theta}  \vert \mathbf{X} )  \\ 
&=\log f(\mathbf{X}, \mathbf{z} \vert \boldsymbol{\theta}) - \log f(\mathbf{z} \vert \mathbf{X}, \boldsymbol{\theta}) - \sum_{k=1}^K \eta_k \Pi_{\textup{KL}}(\mathbf{\Sigma}_k,\mathbf{T}_k) \\
&=  \mathbb{E}_{\mathbf{z}|\mathbf{X}, \boldsymbol{\theta}^{(t)}} \left[\log f (\mathbf{X}, \mathbf{z}|\boldsymbol{\theta}) \right]  - \sum_{k=1}^K \eta_k \Pi_{\textup{KL}}(\mathbf{\Sigma}_k,\mathbf{T}_k) \\
&\qquad \quad - \mathbb{E}_{\mathbf{z}|\mathbf{X}, \boldsymbol{\theta}^{(t)}} \left[ \log f(\mathbf{z}|\mathbf{X}, \boldsymbol{\theta})\right] \\ 
&= Q( \boldsymbol{\theta} \vert \boldsymbol{\theta}^{(t)}) - \mathbb{E}_{\mathbf{z}|\mathbf{X}, \boldsymbol{\theta}^{(t)}} \left[ \log f(\mathbf{z}|\mathbf{X}, \boldsymbol{\theta})\right]  
\end{align*}
Gibb's inequality ensures that 
$$ 
\mathbb{E}_{\mathbf{z}|\mathbf{X}, \boldsymbol{\theta}^{(t)}} [ \log f (\mathbf{z}|\mathbf{X}, \boldsymbol{\theta}) ] \leq \mathbb{E}_{\mathbf{z}|\mathbf{X}, \boldsymbol{\theta}^{(t)}}[ \log f(\mathbf{z}|\mathbf{X}, \boldsymbol{\theta}^{(t)})] .
$$ 
The above then implies that  
$$
\ell_{\boldsymbol{\eta}}( \boldsymbol{\theta}  \vert \mathbf{X} )  - \ell_{\boldsymbol{\eta}}( \boldsymbol{\theta}^{(t)}  \vert \mathbf{X} ) \geq  Q( \boldsymbol{\theta} \vert \boldsymbol{\theta}^{(t)}) - 
Q( \boldsymbol{\theta}^{(t)} \vert \boldsymbol{\theta}^{(t)})
$$
so maximizing $Q( \boldsymbol{\theta} \vert \boldsymbol{\theta}^{(t)})$ 
causes 
$\ell_{\boldsymbol{\eta}}( \boldsymbol{\theta}  \vert \mathbf{X} ) $  
to increase at least as much. This means that we only need 
\[
    Q ( \boldsymbol{\theta}^{(t+1)} \mid \boldsymbol{\theta}^{(t)}) - Q( \boldsymbol{\theta}^{(t)} \mid \boldsymbol{\theta}^{(t)}) \ge 0
\]
for the ascent property to hold. This holds since $\boldsymbol{\theta}^{(t+1)}$ is chosen to maximize $ Q( \boldsymbol{\theta} \vert \boldsymbol{\theta}^{(t)})$.
\end{proof}

\begin{proposition}
The E-step of the regularized EM algorithm remains unchanged, we have $\forall i \in  [\![1,n]\!], \forall k \in  [\![1,K]\!]$ :
\begin{equation}
p_{ik}^{(t)} = \frac{\hat{\pi}_k^{(t)} |\mathbf{\hat{\Sigma}}^{(t)}_k|^{-\frac{1}{2}} e^{-\frac{1}{2} (\mathbf{x}_i-\boldsymbol{\hat{\mu}}_k^{(t)})^\top \mathbf{\hat{\Sigma}}^{t-1}_k (\mathbf{x}_i-\boldsymbol{\hat{\mu}}_k^{(t)})}}{\sum_{j=1}^K \hat{\pi}_j^{(t)} |\mathbf{\hat{\Sigma}}^{t}_j|^{-\frac{1}{2}} e^{-\frac{1}{2} (\mathbf{x}_i-\boldsymbol{\hat{\mu}}_j^{(t)})^\top \mathbf{\hat{\Sigma}}^{t-1}_j (\mathbf{x}_i-\boldsymbol{\hat{\mu}}_j^{(t)})}} 
\end{equation}
\end{proposition}
\begin{proof} 
This is obvious since the penalization term  has no impact on the expected conditional likelihood. Thus evaluating  $Q( \boldsymbol{\theta} \vert \boldsymbol{\theta}^{(t)})$ one can notice that, similar to non-penalized case, only 
$$
p_{ik}^{(t)} = \Pr(z_j=k| \mathbf{x}_i,\boldsymbol{\theta}^{(t)})
$$
depends on $\boldsymbol{\theta}^{(t)}$. One can then apply the Bayes theorem and use the fact that conditional class distributions are Gaussian, i.e.,  $\mathbf{x}_i | z_j = k \sim \mathcal N_m(\boldsymbol{\mu}_k,\boldsymbol{\Sigma}_k)$ and $\pi_k = \Pr(z_j = k)$.   
\end{proof}

\begin{proposition}
The M-step consists in :
\begin{align*}
\pi_k^{(t+1)} &= \frac{1}{n} \sum_{i=1}^n p_{ik}^{(t)}\,, &
\boldsymbol{\hat{\mu}}_k^{(t+1)} 
= \sum_{i=1}^n w_{ik}^{(t)} \mathbf{x}_i
\end{align*}
\begin{equation*}
\hat{\boldsymbol{\Sigma}}_k^{(t+1)} = \beta_k^{(t+1)} \sum_{i=1}^n w_{ik}^{(t)}  (\mathbf{x}_i-\boldsymbol{\hat{\mu}}_k^{(t)})( \mathbf{x}_i-\hat{\boldsymbol{\mu}}_k^{(t)})^\top + (1-\beta_k^{(t+1)})\mathbf{T}_k, 
\end{equation*}
where
$\displaystyle \beta_k^{(t+1)} = \frac{n\pi_k^{(t+1)}}{\eta_k + n\pi_k^{(t+1)}}$ and $\displaystyle w_{ik}^{(t)} = \frac{p_{ik}^{(t)}}{\sum_{i=1}^n p_{ik}^{(t)}}
$
\end{proposition}
\begin{proof}
We wish to maximize 
\begin{align*}
 Q( \boldsymbol{\theta} \vert \boldsymbol{\theta}^{(t)}) &= \sum_{i=1}^n \sum_{k=1}^K p_{ik}^{(t)} ( \log f(\mathbf{x}_i \in \mathcal{C}_k | \boldsymbol{\theta}) + \log \pi_k ) \\
&- \frac{1}{2} \sum_{k=1}^K \eta_k \left( \mathrm{tr}(\mathbf{\Sigma}_k^{-1}\mathbf{T}_k) - \log |\mathbf{\Sigma}_k^{-1}\mathbf{T}_k| - m\right) 
\end{align*}
Same steps are used as in the usual EM algorithm to derive $\hat \pi_k^{(t+1)}$ and $\hat{\boldsymbol{\mu}}_k^{(t+1)}$. Solving for $\boldsymbol{\hat{\Sigma}}_k^{(t+1)}$, we can compute the solution of $ \nabla_{\boldsymbol{\Sigma}_k^{-1}}  Q(\boldsymbol{\theta} \vert \boldsymbol {\theta}^{(t)}) = \mathbf{0}$ while keeping the other parameters of $\boldsymbol{\theta}$ fixed with their solutions for  iteration $t$. 
By setting $\tilde{\mathbf{x}}_i= \mathbf{x}_i-\boldsymbol{\hat{\mu}}_k^{(t)}$, this leads to solving the following estimating equation:
\begin{align*}
&\sum_{i=1}^n p_{ik}^{(t)} \left(\boldsymbol{\Sigma}_k   - \tilde{\mathbf{x}}_i\tilde{\mathbf{x}}_i^\top\right)   -  \eta_k \left( \boldsymbol{\Sigma}_k - \mathbf{T}_k \right)  = \mathbf{0} \\
\Leftrightarrow  &\sum_{i=1}^n p_{ik}^{(t)} \tilde{\mathbf{x}}_i\tilde{\mathbf{x}}_i^\top + \eta_k \mathbf{T}_k  = \Big( \eta_k + \sum_{i=1}^n p_{ik}^{(t)} \Big) \mathbf{\Sigma}_k
\end{align*}
whose solution is $\hat{\boldsymbol{\Sigma}}_k^{(t+1)}$ given in the proposition. 
\end{proof}

\begin{algorithm}[!t]
\caption{$L$-fold cross-validation of $\eta_k$ for $k$th cluster} \label{alg:CV2}
\textbf{Input:}  Initial set of indices $\mathcal{D}^0$ and scale $\hat \theta_k^0$ for the $k$th cluster. A set $\{ \eta_j \}_{j=1}^J$ of candidate penalty parameter values. 
\begin{algorithmic}[1]
\State Split $\mathcal D^0$ into $L$ distinct folds $\mathcal{D}_1,\ldots,\mathcal{D}_L$ s.t. $\mathcal D^0 = \cup_{l=1}^L \mathcal D_l$ and set $\mathbf{T}_k^0 = \hat \theta_k^0 \cdot \mathbf{I}_m$ as the target matrix
\State Set $\mathrm{Err}_j \equiv \mathrm{Err}(\eta_j)=0$ for $j=1,\ldots,J$.

    \For{$l \in [\![1,L]\!]$}
        \State Set $ \mathcal D_{\textup{val}} = \mathcal{D}_l$ and $\mathcal D_{\textup{tr}} = \mathcal{D} /\ \mathcal{D}_{\textup{val}}$
         \State $\mathbf{S}_{\textup{val}} = \frac 1{|\mathcal{D}_{\textup{val}} |} \sum_{i \in \mathcal D_{\textup{val}} }  (\mathbf{x}_i-\bar{\mathbf{x}}_{\textup{val}})(\mathbf{x}_i- \bar{\mathbf{x}}_{\textup{val}})^\top$
        \State $\boldsymbol{\hat{\Sigma}}= \frac 1{|\mathcal{D}_{\textup{tr}} |} \sum_{i \in  \mathcal{D}_{\textup{tr}}} (\mathbf{x}_i-\bar{\mathbf{x}}_{\textup{tr}})(\mathbf{x}_i-\bar{\mathbf{x}}_{\textup{tr}})^\top$
       \For{$\eta \in \{\eta_1,\ldots,\eta_J\}$}
            \State $\hat{\boldsymbol{\Sigma}}_{\eta}= \frac{|\mathcal D_{\textup{tr}}|}{\eta+|\mathcal D_{\textup{tr}}|} \boldsymbol{\hat{\Sigma}} + \frac{\eta}{\eta+|\mathcal D_{\textup{tr}}|} \mathbf{T}_k^0$
            \State $\mathrm{Err}_{l} = \mathrm{Err}_{l} + \mathrm{tr}(\hat{\boldsymbol{\Sigma}}_{\eta}^{-1}\mathbf{S}_{\textup{val}}) + \log | \boldsymbol{\Sigma}_{\eta} | $
        \EndFor
 \EndFor
\State Choose $\eta_j$ that minimizes $ \{\mathrm{Err}(\eta_j)\}_{j=1}^J$ 

\end{algorithmic}
\end{algorithm}


Let us now turn the discussion to user-defined  regularization parameters, $\mathbf{T}_k$-s and $\eta_k$-s. 
The fixed SPD target matrices $\mathbf{T}_k$ can  bring prior knowledge to the estimation problem or simply provide a well-conditioned estimator of cluster covariance matrices by shrinking the covariance matrix updates towards a scaled identity matrix. In the latter case, the obvious target matrices to be used are 
$\mathbf{T}_k = \hat \theta_k^0  \cdot \mathbf{I}_m$
where $ \hat \theta_k^0$ is initial estimate of the scale statistic $\theta_k = \mathrm{tr}(\boldsymbol{\Sigma}_k)/m$. 
Here we use $\hat \theta_k^0 = \mathrm{tr}(\hat{\boldsymbol{\Sigma}}_k^0)/m$, where $\hat{\boldsymbol{\Sigma}}_k^0$ is an initial estimate of the $k$th cluster covariance matrix obtained \textit{e.g.}, with a first clustering obtained by K-means or after several iterations of the EM.

The choice of regularization parameter is also important. We choose the regularization parameter using a cross-validation procedure minimizing the negative Gaussian log-likelihood as in \cite{yi2020shrinking}. Each $\eta_k$ is estimated independently among a set of candidates values $\{\eta_1,\ldots,\eta_J\}$. The procedure is described in Algorithm~\ref{alg:CV2}.

\section{Experiments on synthetic data}
\label{sec:3}

\subsection{Simulation set-up}

The performance of the proposed method is compared with K-means and vanilla Gaussian-EM (G-EM) algorithm. In order to make G-EM competitive in regimes with high dimensions and low sample size, we add a vanilla regularization to G-EM for the estimation of the covariance matrix. K-means is implemented using Scikit-learn library \cite{scikit-learn} while the proposed regularized Gaussian EM (RG-EM) algorithm is implemented from scratch. For both the G-EM and the proposed RG-EM, we grant a maximum of 200 iterations for parameter estimation. For our method, we recompute the optimal $\eta_k$ every 20 iterations.
\begin{figure}[!t]
\centering
\subfigure[$n=3000$\label{fig:1a}]{\includegraphics[width=4.2cm]{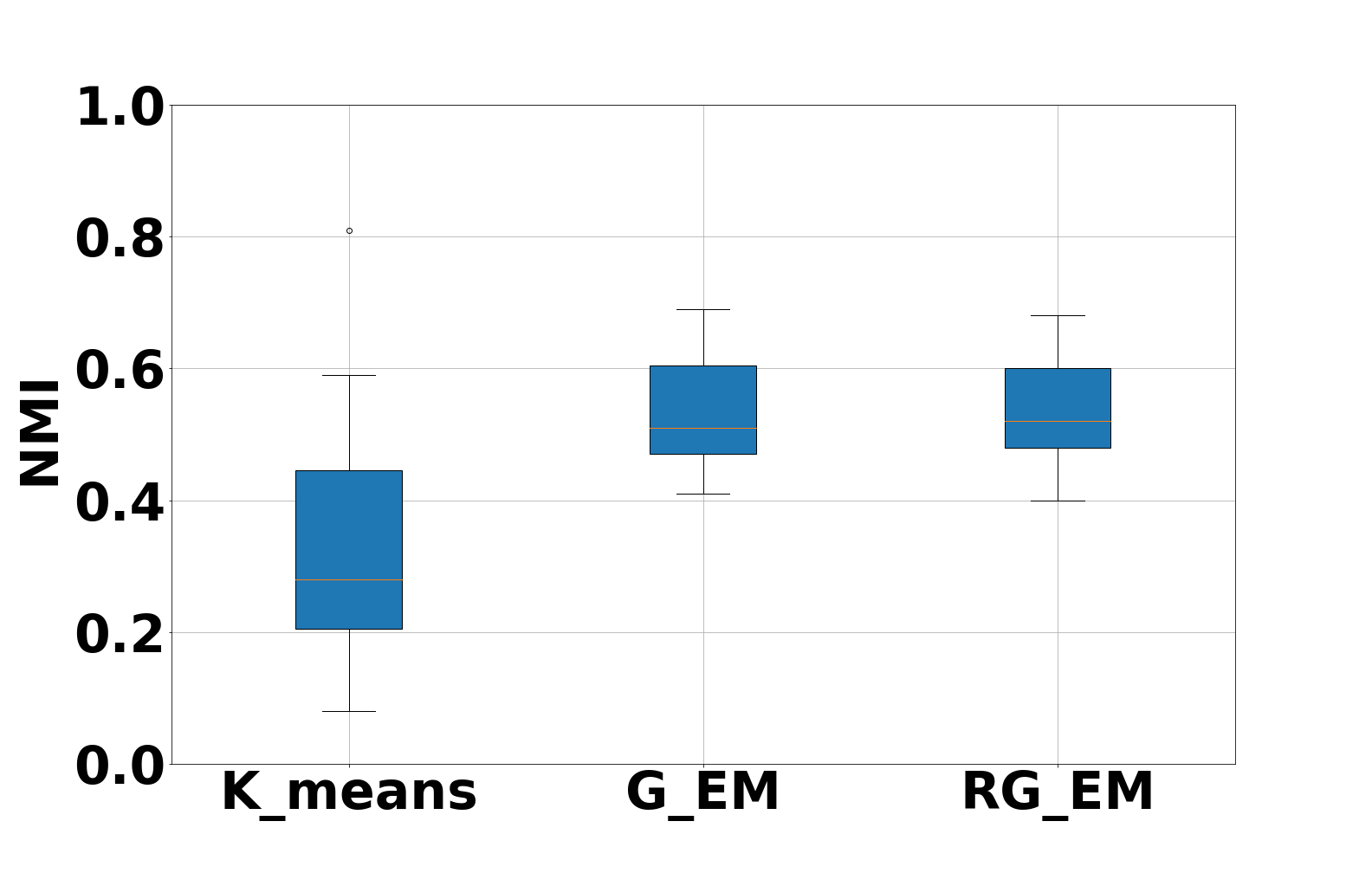}}
\subfigure[$n=1000$\label{fig:1a}]{\includegraphics[width=4.2cm]{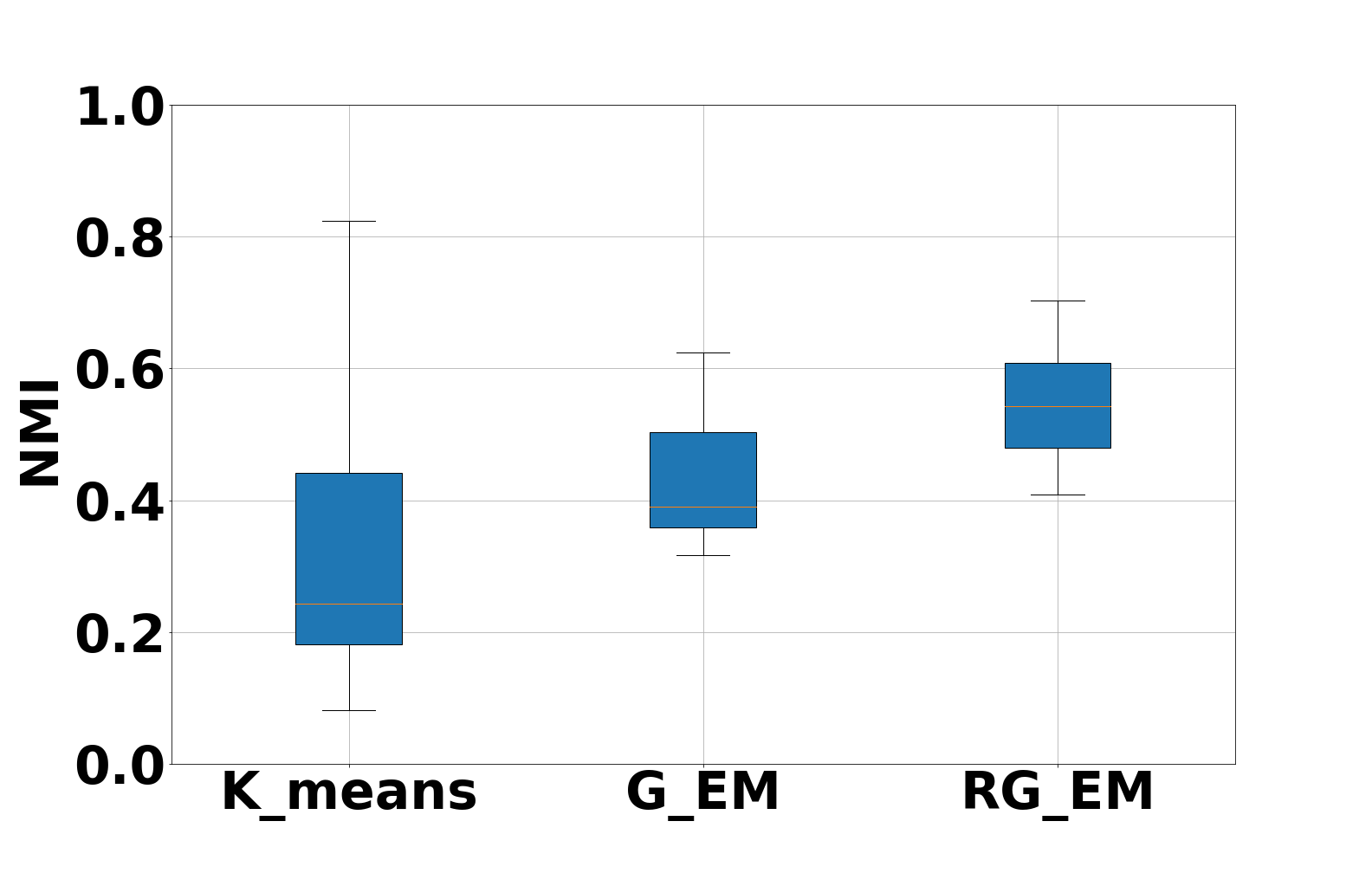}}
\subfigure[$n=600$\label{fig:1a}]{\includegraphics[width=4.2cm]{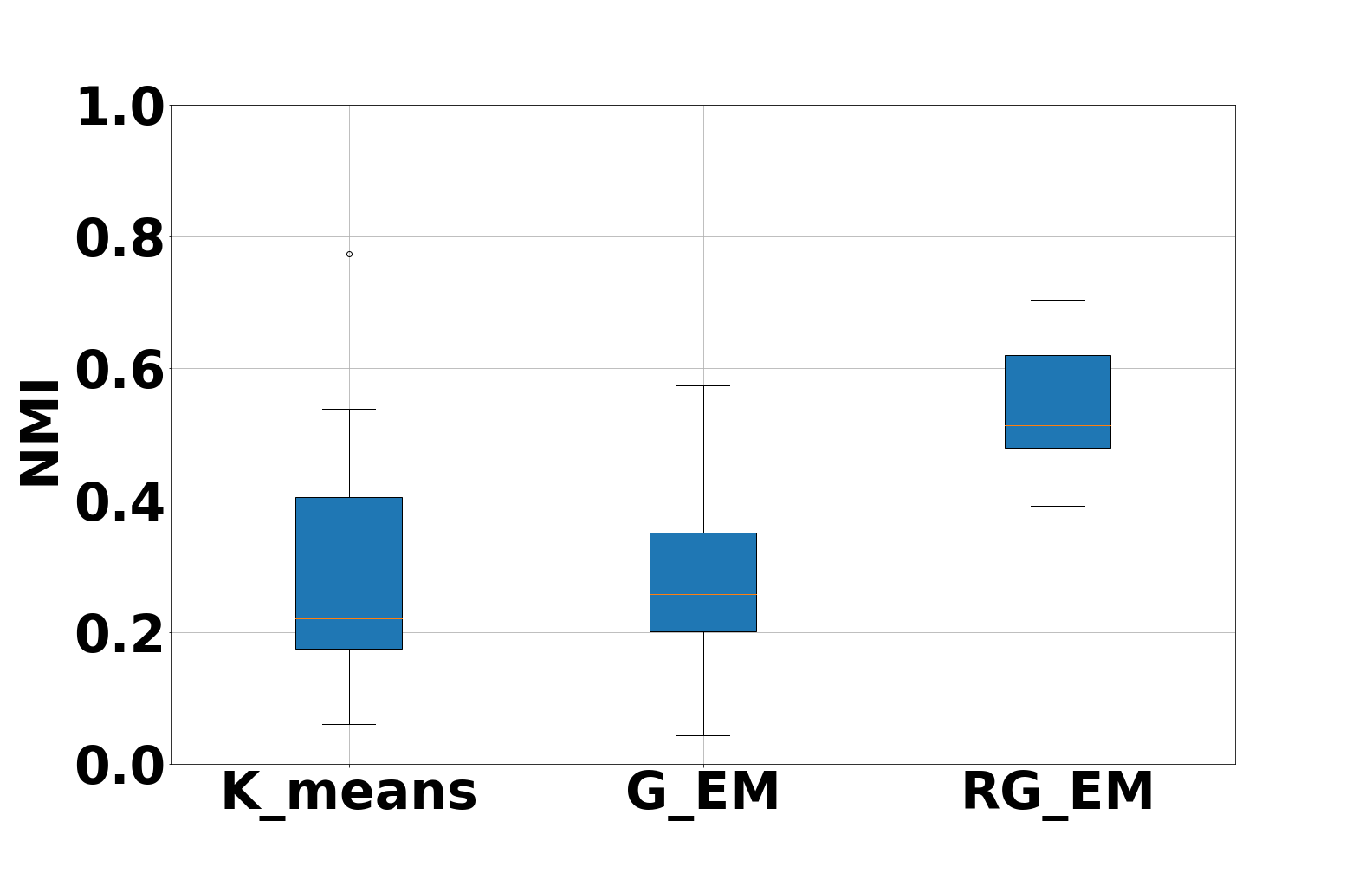}}
\subfigure[$n=400$\label{fig:1a}]{\includegraphics[width=4.2cm]{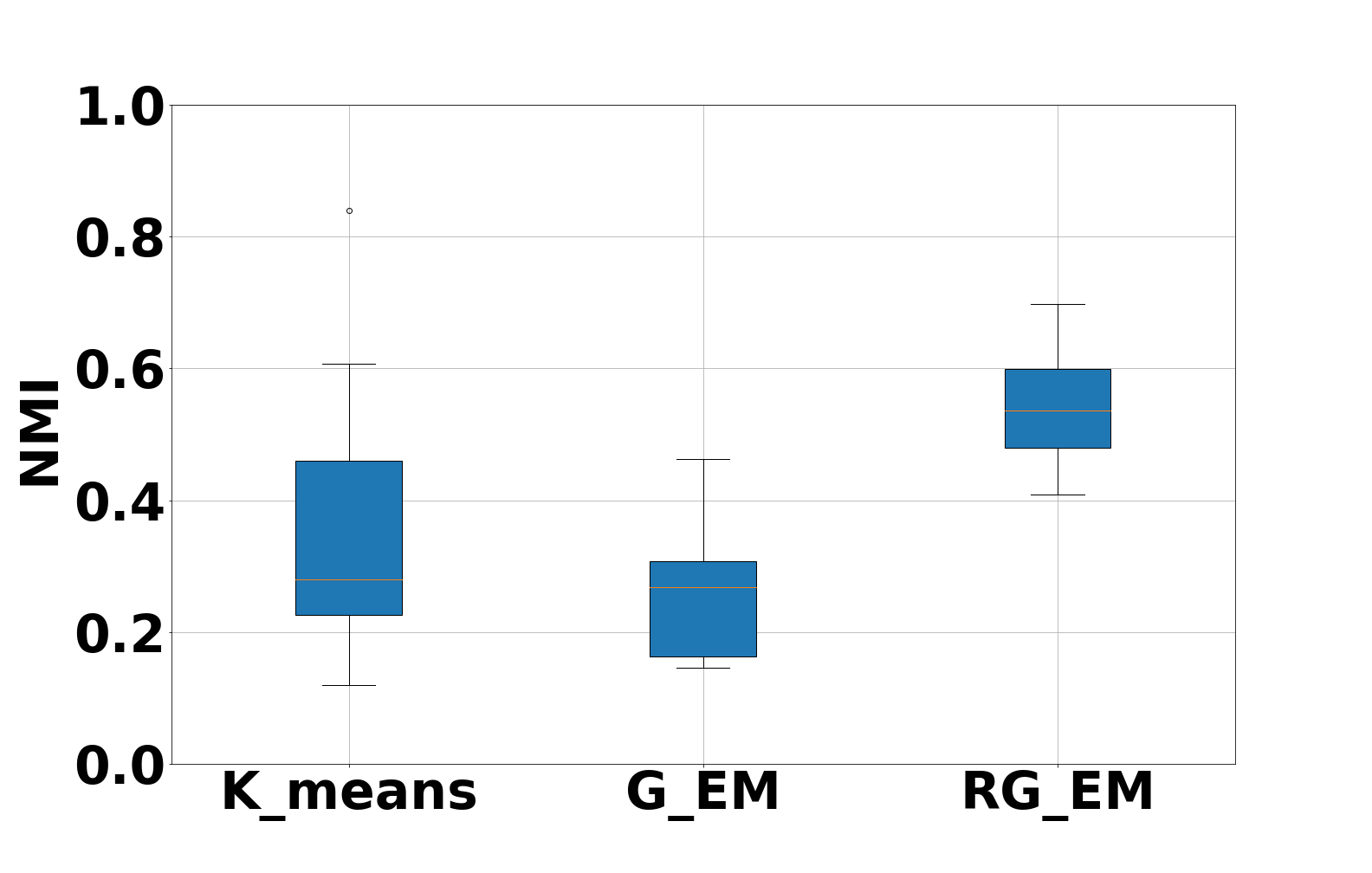}}
\caption{Performances in terms of NMI for different sample lengths $n$. The dimension is $m=50$.}
\label{fig:1} 
\end{figure}


The number of clusters is $K=3$, and data is generated from GMM with equal prior probabilities ($\pi_k=1/3$, $\forall k$).  The mean vectors of each cluster are drawn on a centered sphere of radius $2$. Cluster covariance matrices possess autoregressive AR(1) structure $(\boldsymbol{\Sigma}_k)_{ij}= \theta_k \varrho_k^{|i-j|}$, where each  cluster having its own AR correlation coefficient $\varrho_k$. The total number of points and the dimension vary across the simulations.

\subsection{Performance in various settings}

We test the performances of each method on various scenarios differing by the number of points $n$ and dimension $m$. For each scenario, we run 25 simulations with different covariance matrices and mean vectors.

Results are displayed in Fig.~1, and performance is evaluated using NMI index \cite{McDaid2011Normalized}. Computed using both the mutual information and the entropy of each cluster, it is a score between 0 and 1 that reflects how similar the set found are compared to the original clusters. In the first scenario, a lot of samples are available for each cluster. K-means performs poorly compared to EM-based methods, probably because clusters sometimes overlap. G-EM and proposed RG-EM perform equally well since there are enough data to accurately estimate the covariance matrices (regularization is not necessary). Then we gradually reduce $n$ in the train set. As can be noted, the performance of K-means is not much affected, while the performance of G-EM drops heavily since there is not enough data for obtaining well-conditioned covariance matrix updates, although the vanilla regularization ensures invertibility. On the contrary, RG-EM handles very well the reduction of the train set size and keeps almost equal performances, using the target matrix to compensate for the lack of data. 

Fig.2 displays the Frobenius error for estimation of the covariance matrix of cluster 1 across the iterations for $n=600$ and $m=50$.
As expected, the Robust Gaussian EM method has a much smaller error compared to the vanilla EM method. It converges much faster with a smaller variance toward its final estimation. On the contrary, vanilla EM estimation converges toward a worse estimation of the covariance matrix, as the error is higher than the first guess with K-means.
\begin{figure}[!t]
\centering
\subfigure{\includegraphics[width=7.5cm,trim=10 1 10 100, clip]{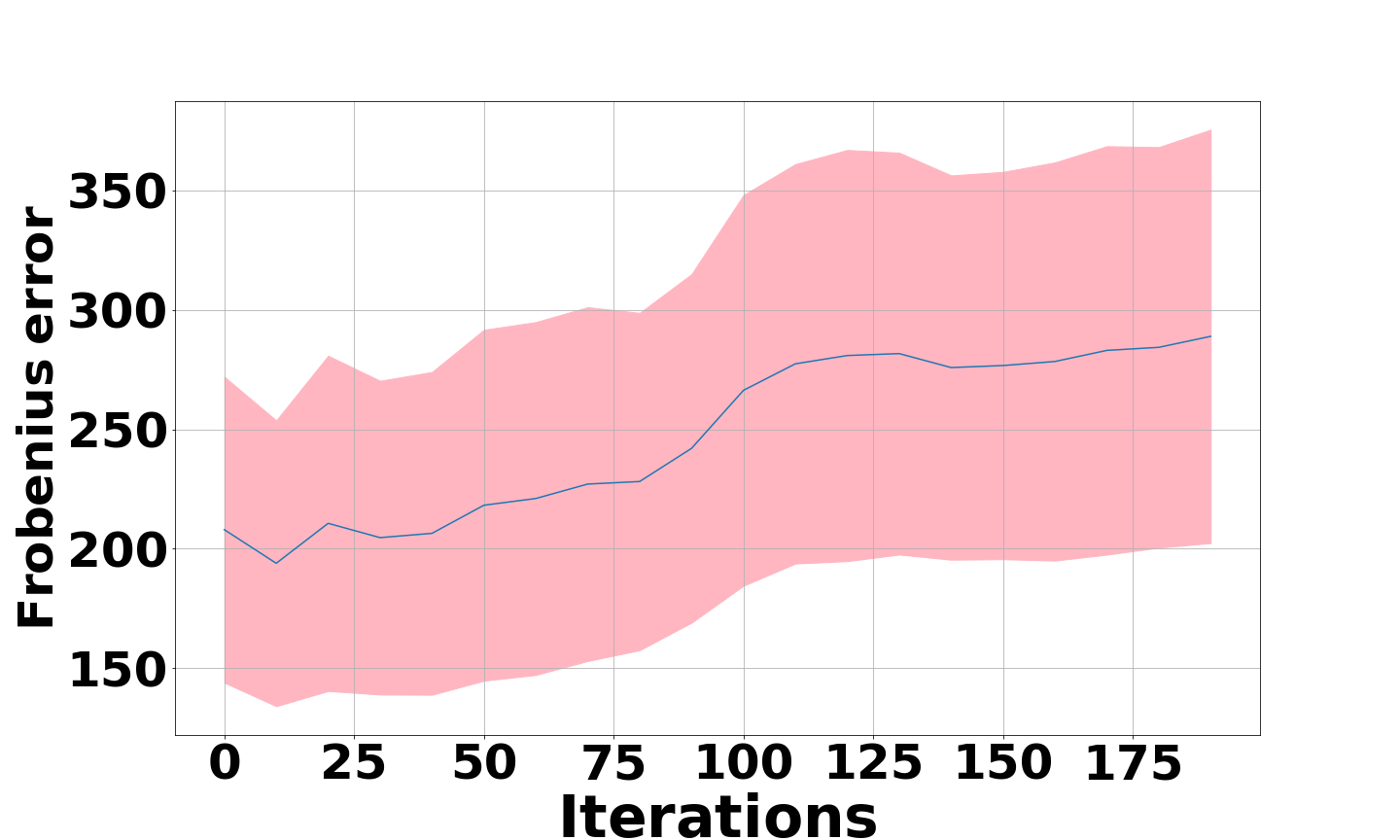}}
\subfigure{\includegraphics[width=7.5cm,trim=10 1 10 100, clip]{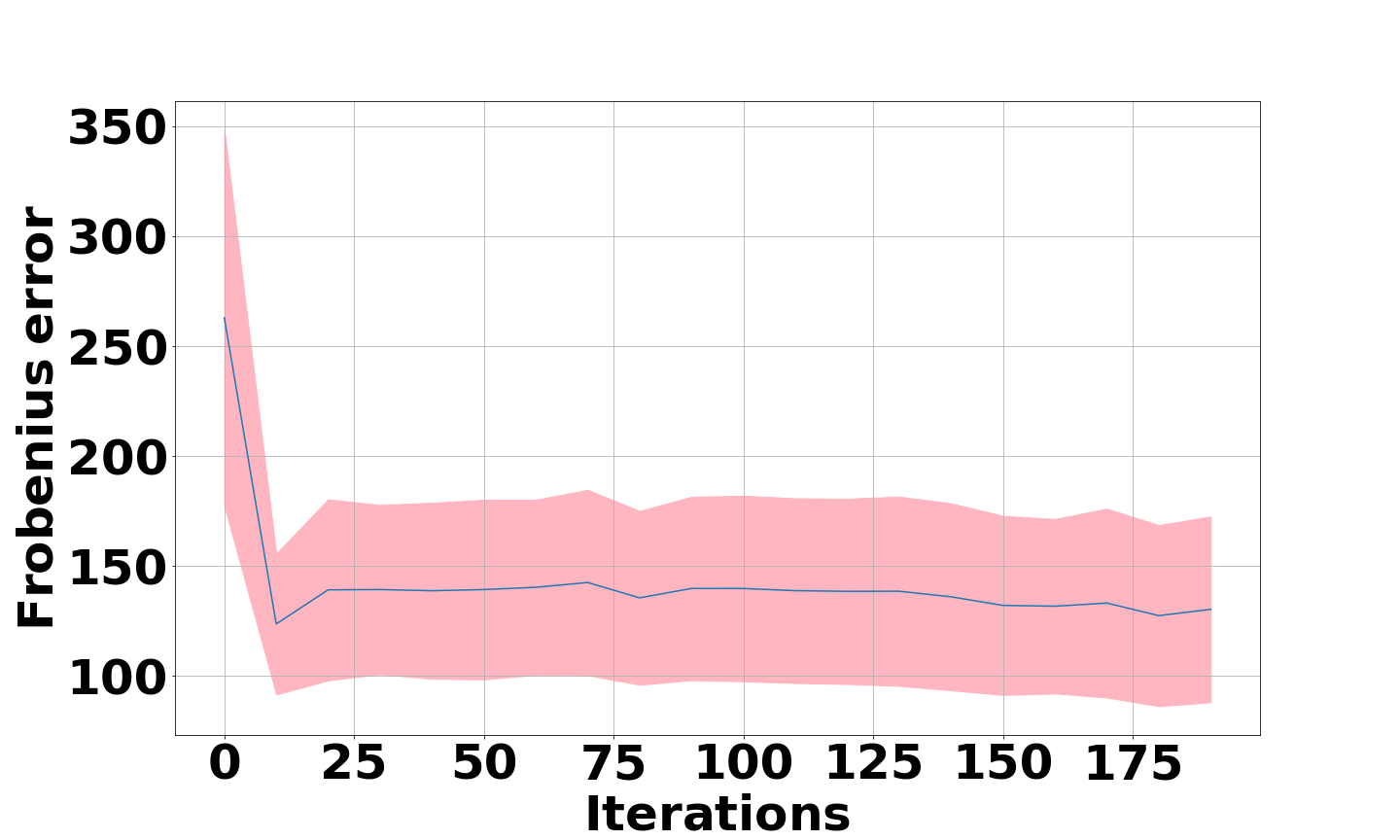}}
\caption{Frobenius error for G-EM ({\it top panel}) and proposed RG-EM ({\it bottom panel}) for cluster 1; $n=600$, $m=50$.}
\label{fig:2}
\end{figure}


We conclude our experiments by studying the evolution of NMI of each method when the dimension $m$ increases for fixed values of $n=600$ and $n=1000$. As can be noted from Fig.~3, while RG-EM can cope well with the high-dimensionality of the data, with performances almost as good as in low dimensions, there is a breakpoint dimension for G-EM where its NMI drops. Such a threshold value depends on the number of data points: the larger the train set is, the higher the breakpoint is.

\begin{figure}[!t]
\centering
\subfigure{\includegraphics[width=7.5cm,trim=10 1 10 110, clip]{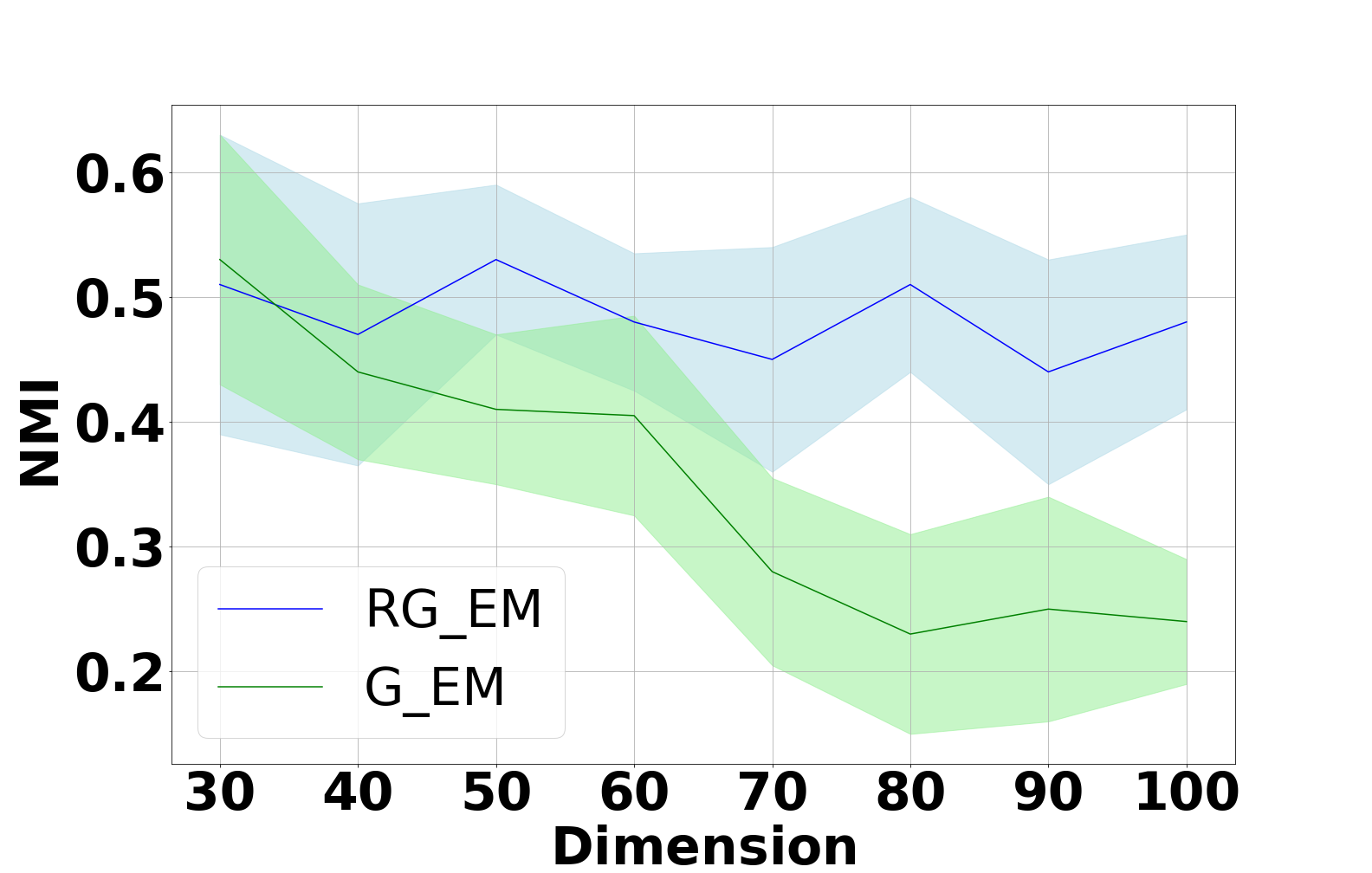}}
\subfigure{\includegraphics[width=7.5cm,trim=10 1 10 110, clip]{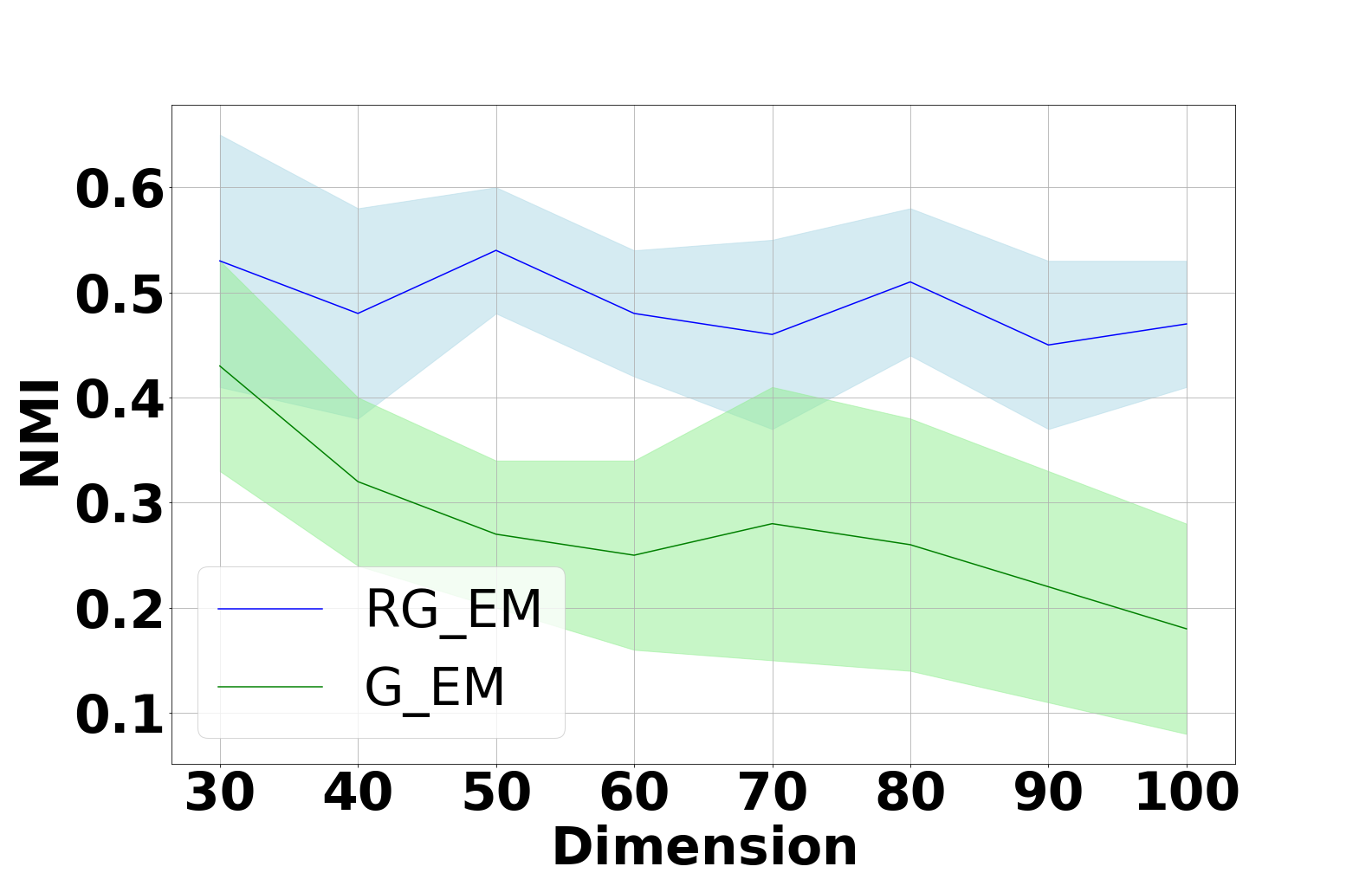}}
\caption{Performance evaluation via NMI index for different dimensions $m$ when $n=1000$ ({\it top panel}) and $n=600$ ({\it bottom panel}).}
\label{fig:3} 
\end{figure}

\section{Conclusion}
\label{sec:conclusions}

In this paper, we presented a regularized version of the EM algorithm for GMM that outperforms several state-of-the-art methods in low sample size regimes on simulated datasets. In this new approach, the estimation of the covariance matrix is regularized with an additive penalty that shrinks the covariance matrix towards a preset target matrix. If the target matrix is close enough to the actual covariance matrix, we achieve very good performances even with few data in high dimensions. The optimal coefficients $\eta_k$ controlling the regularization are chosen following a cross-validation procedure and updated regularly across the iterations. Such cross-validation procedure avoids regularization if enough data are available, retrieving the performance of the classical EM algorithm. Our method can thus be seen as an enhancement of regular EM when the data aspect ratio $\frac{n}{m}$ is low. 

\bibliographystyle{IEEEtran}
\bibliography{IEEEabrv,IEEEbib}
\end{document}